\documentclass{article}
\usepackage{amsmath,graphicx,mlspconf}
\usepackage{microtype}
\usepackage{graphicx}
\usepackage{booktabs} 
\usepackage{mathtools} 
\usepackage{amssymb}
\usepackage{amsthm}
\usepackage{hyperref}
\usepackage{subcaption}
\usepackage{wrapfig}
\usepackage{adjustbox}
\usepackage{url}
\usepackage{color}
\usepackage[inline]{enumitem}

\usepackage[numbers]{natbib}
\usepackage{hyperref}

\newtheorem{thm}{Theorem}

\newcommand{\diag}{\mathrm{diag}}

\newcommand{\N}{\mathbb{N}}

\usepackage{amsmath,amsfonts,bm}

\def\eqref#1{equation~\ref{#1}}

\def\1{\bm{1}}

\def\vtheta{{\bm{\theta}}}

\def\vc{{\bm{c}}}

\def\vh{{\bm{h}}}

\def\vo{{\bm{o}}}

\def\vx{{\bm{x}}}
\def\vy{{\bm{y}}}

\def\mD{{\bm{D}}}

\def\mF{{\bm{F}}}

\def\mI{{\bm{I}}}

\def\mM{{\bm{M}}}

\def\mP{{\bm{P}}}
\def\mQ{{\bm{Q}}}

\def\mS{{\bm{S}}}

\def\mU{{\bm{U}}}

\def\mW{{\bm{W}}}
\def\mX{{\bm{X}}}

\def\mPhi{{\bm{\Phi}}}
\def\mLambda{{\bm{\Lambda}}}

\DeclareMathAlphabet{\mathsfit}{\encodingdefault}{\sfdefault}{m}{sl}
\SetMathAlphabet{\mathsfit}{bold}{\encodingdefault}{\sfdefault}{bx}{n}

\def\gG{{\mathcal{G}}}

\newcommand{\R}{\mathbb{R}}

\def\mPhi{{\bm{\Phi}}}
\def\mPsi{{\bm{\Psi}}}
\def\mTheta{{\bm{\Theta}}}

\renewcommand\paragraph[1]{\subsubsection{#1}}
\renewcommand\citep[1]{\cite{#1}}
\renewcommand\citet[1]{\cite{#1}}
\copyrightnotice{978-1-7281-6338-3/21/\$31.00 {\copyright}2021 IEEE}

\toappear{2021 IEEE International Workshop on Machine Learning for Signal Processing, Oct.\ 25--28, 2021, Gold Coast, Australia}

\title{Data-driven Learning of Geometric Scattering Modules for GNNs}
\name{
    Alexander Tong$^{1 *}$%
    \quad Frederick Wenkel$^{2 *}$%
    \quad Kincaid Macdonald$^{3}$%
    \quad Smita Krishnaswamy$^{4, 1 \dagger}$%
    \quad Guy Wolf$^{2 \dagger}$\thanks{$^{*}$ Equal contribution. $^{\dagger}$ Equal senior author contribution. This research was partially funded by IVADO grant PRF-2019-3583139727, FRQNT grant 299376, CIFAR AI Chair [\emph{G.W.}]; CZI grants 182702 \& CZF2019-002440 [\emph{S.K.}]; and NIH grants R01GM135929 \& R01GM130847 [\emph{G.W., S.K.}]. The content provided here is solely the responsibility of the authors and does not necessarily represent the official views of the funding agencies. Correspondence to: \texttt{smita.krishnaswamy@yale.edu} and \texttt{guy.wolf@umontreal.ca}}%
}
\address{%
    $^{1}$ Yale University, Dept. of Comp. Sci.; $^{3}$ Dept. of Math.; $^{4}$ Dept. of Genetics, New Haven, CT, USA \\%
    $^{2}$ Universit\'{e} de Montr\'{e}al, Dept. of Math. \& Stat.; Mila -- Quebec AI Institute, Montreal, QC, Canada\\%
}

\begin{document}

\maketitle

\begin{abstract} 
We propose a new graph neural network (GNN) module, based on relaxations of recently proposed geometric scattering transforms, which consist of a cascade of graph wavelet filters. Our learnable geometric scattering (LEGS) module enables adaptive tuning of the wavelets to encourage band-pass features to emerge in learned representations. The incorporation of our LEGS-module in GNNs enables the learning of longer-range graph relations compared to many popular GNNs, which often rely on encoding graph structure via smoothness or similarity between neighbors. Further, its wavelet priors result in simplified architectures with significantly fewer learned parameters compared to competing GNNs. We demonstrate the predictive performance of LEGS-based networks on graph classification benchmarks, as well as the descriptive quality of their learned features in biochemical graph data exploration tasks.

\end{abstract}

\begin{keywords}
Geometric deep learning, graph neural networks, geometric scattering
\end{keywords}
\section{Introduction}\label{sec: introduction}
Geometric deep learning has recently emerged as an increasingly prominent branch of deep learning~\citep{bronstein_geometric_2017}. At the core of geometric deep learning is the use of graph neural networks (GNNs) in general, and graph convolutional networks (GCNs) in particular, which ensure neuron activations follow the geometric organization of input data by propagating information across graph neighborhoods~\citep{kipf_semi-supervised_2016, hamilton_inductive_2017, xu_how_2019}
. However, recent work has shown the difficulty in generalizing these methods to more complex structures, identifying common problems and phrasing them in terms of oversmoothing~\citep{li2018deeper}, oversquashing~\citep{alon2020bottleneck} or underreaching~\citep{barcelo2020logical}. 


Recently, an alternative approach was presented to provide deep geometric representation learning by generalizing Mallat's scattering transform \citep{mallat_group_2012}, originally proposed to provide a mathematical framework for understanding convolutional neural networks, to graphs~\citep{gao2019geometric, gama2019diffusion, zou_graph_2019} and manifolds~\citep{perlmutter_geometric_2018}. 
Similar to traditional scattering, which can be seen as a convolutional network with non-learned wavelet filters, geometric scattering is defined as a GNN with handcrafted graph filters,  constructed as diffusion wavelets over the input graph~\citep{coifman_diffusion_2006-1}, which are then cascaded with pointwise absolute-value nonlinearities.
The efficacy of geometric scattering features in graph processing tasks was demonstrated in~\citet{gao2019geometric}, with both supervised learning and data exploration applications. Moreover, their handcrafted design enables rigorous study of their properties, such as stability to deformations and perturbations, and provides a clear understanding of the information extracted by them, which by design (e.g., the cascaded band-pass filters) goes beyond low frequencies to consider richer notions of regularity~\citep{gama_stability_2019, perlmutter2019understanding}. 

However, while geometric scattering transforms provide effective universal feature extractors, their handcrafted design does not allow the automatic task-driven representation learning that is so successful in traditional GNNs and neural networks in general. Here, we combine both frameworks by incorporating richer multi-frequency band features from geometric scattering into GNNs, while allowing them to be flexible and trainable. We introduce the geometric scattering module, which can be used within a larger neural network.
We call this a {\em learnable geometric scattering (LEGS) module} and show it inherits properties from the scattering transform while allowing the scales of the diffusion to be learned.
\section{Preliminaries: Geometric Scattering}\label{sect_geometric scattering} 

Let $\gG = (V,E,w)$ be a weighted graph with $V\coloneqq \{v_1,\dots,v_n\}$ the set of nodes, $E\subset \{\{v_i, v_j\}\in V\times V , i\neq j\}$ the set of (undirected) edges and $w : E \to (0,\infty)$ assigning (positive) edge weights to the graph edges.
We define a \textit{graph signal} as a function $x: V \rightarrow \R$ on the nodes of $\gG$ and aggregate them in a signal vector $\vx\in \R^n$ with the $i^{th}$ entry being $x(v_i)$.
We define the \textit{weighted adjacency matrix} $\mW\in\R^{n\times n}$ of $\gG$ as
    $\mW[v_i,v_j] \coloneqq
    w(v_i,v_j) \text{ if } \{v_i,v_j\}\in E, 
    \text{ and } 0 \text{ otherwise}$
and the \textit{degree matrix} $\mD\in\R^{n\times n}$ of $\gG$ as $\mD\coloneqq \diag(d_1,\dots, d_n)$ with $d_i\coloneqq \deg(v_i)\coloneqq \sum_{j=1}^n \mW[v_i,v_j]$ the \textit{degree} of node $v_i$.

The geometric scattering transform~\citep{gao2019geometric} consists of a cascade of graph filters constructed from a left stochastic diffusion matrix $\mP \coloneqq \frac{1}{2} \big( \mI_n + \mW \mD^{-1} \big)$, which corresponds to transition probabilities of a lazy random walk Markov process. The laziness of the process signifies that at each step it has equal probability of staying at the current node or transitioning to a neighbor.
Scattering filters are defined via graph-wavelet matrices $\mPsi_j\in\R^{n\times n}$ of order $j\in\N_0$, as
\begin{align}\label{eq_wavelet matrix}
    \mPsi_0 &\coloneqq \mI_n - \mP, \nonumber \\
    \mPsi_j &\coloneqq \mP^{2^{j-1}} - \mP^{2^j} = \mP^{2^{j-1}} \big( \mI_n - \mP^{2^{j-1}} \big), \quad j\geq 1.
\end{align}
These diffusion wavelet operators partition the frequency spectrum into dyadic frequency bands, which are then organized into a full wavelet filter bank $\mathcal{W}_J\coloneqq\{\mPsi_j, \mPhi_J\}_{0\leq j\leq J}$, where $\mPhi_J\coloneqq \mP^{2^J}$ is a pure low-pass filter, similar to the one used in GCNs. It is easy to verify that the resulting wavelet transform is invertible, since a simple sum of filter matrices in $\mathcal{W}_J$ yields the identity. Moreover, as discussed in~\citet{perlmutter2019understanding}, this filter bank forms a nonexpansive frame, which provides energy preservation guarantees as well as stability to perturbations, and can be generalized to a wider family of constructions that encompasses the variations of scattering transforms on graphs from~\citet{gama2019diffusion,gama_stability_2019} and~\citet{zou_graph_2019}.

Given the wavelet filter bank $\mathcal{W}_J$, node-level scattering features are computed by stacking cascades of bandpass filters and element-wise absolute value nonlinearities to form 
\begin{equation}\label{eq_scattering (node) features}
     \mU_p \vx \coloneqq \mPsi_{j_m} \vert \mPsi_{j_{m-1}} \dots \vert \mPsi_{j_2} \vert \mPsi_{j_1}\vx\vert \vert \dots \vert,
\end{equation}
parameterized by the scattering path $p \coloneqq (j_1, \dots, j_m)\in \cup_{m \in \N} \N_0^{m}$ that determines the filter scales of each wavelet. Whole-graph representations are obtained by aggregating node-level features via statistical moments over the nodes of the graph~\citep{gao2019geometric}, which yields the geometric scattering features
\begin{equation}\label{eq_scattering (graph) featrues}
    \mS_{p,q} \vx \coloneqq \sum_{i=1}^n \vert \mU_p \vx [v_i] \vert^q,
\end{equation}
indexed by the scattering path $p$ and moment order $q$. Finally, we note that it can be shown that the graph-level scattering transform $\mS_{p,q}$ guarantees node-permutation invariance, while $\mU_{p}$ is permutation equivariant~\citep{perlmutter2019understanding,gao2019geometric}.

\section{Adaptive Geom. Scattering Relaxation}\label{sect_theory} 

The geometric scattering construction, described in Sec.~\ref{sect_geometric scattering}, can be seen as a particular GNN architecture with handcrafted layers, rather than learned ones. This provides a solid mathematical framework for understanding the encoding of geometric information in GNNs~\citep{perlmutter2019understanding}, while also providing effective unsupervised graph representation learning for data exploration, which also has advantages in supervised learning tasks~\citep{gao2019geometric}.
While the handcrafted design in~\citet{perlmutter2019understanding} and \citet{gao2019geometric} is not a priori amenable to task-driven tuning provided by end-to-end GNN training, we note that the cascade in Eq.~\ref{eq_scattering (graph) featrues} does conform to a neural network architecture suitable for backpropagation.
Therefore, in this section, we show how and under what conditions a relaxation of the selection of the scales preserves some of the useful mathematical properties established in~\citet{perlmutter2019understanding}.

We replace the handcrafted dyadic scales in Eq.~\ref{eq_wavelet matrix} with an adaptive monotonic sequence of integer diffusion time scales $0 < t_1 < \cdots < t_J$, which are selected via training. The adaptive filter bank $\mathcal{W}_J^\prime \coloneqq\{\mPsi_j^\prime, \mPhi_J^\prime\}_{j=0}^{J-1}$, contains wavelets
\begin{align}\label{eq_adaptive wavelet matrices}
    \mPsi_0^\prime &\coloneqq  \mI_n - \mP^{t_1}, \quad \mPhi_J^\prime \coloneqq \mP^{t_J},
     \\
    \mPsi_j^\prime &\coloneqq \mP^{t_j} - \mP^{t_{j+1}}, \quad 1 \leq j \leq J-1. \nonumber
\end{align}
The following theorem shows that for any selection of scales, the relaxed construction of $\mathcal{W}_J^\prime$ yields a nonexpansive frame, similar to the result from~\citet{perlmutter2019understanding} shown for the original handcrafted construction.
\begin{thm}\label{thm:frame}
There exists a constant $C > 0$ that only depends on $t_1$ and $t_J$ such that for all $\vx\in L^2(\gG,\mD^{-1/2})$,
$$
    C \Vert \vx \Vert_{\mD^{-\frac{1}{2}}}^2 \leqslant \Vert \mPhi_J^\prime \vx \Vert_{\mD^{-\frac{1}{2}}}^2 + \sum_{j=0}^J\Vert \mPsi_j^\prime \vx \Vert_{\mD^{-\frac{1}{2}}}^2 \leqslant \Vert \vx \Vert_{\mD^{-\frac{1}{2}}}^2 ,
$$
where the norm is the one induced by the space $L^2(\gG,\mD^{-1/2})$.
\end{thm}

\begin{proof}
Note that $\mP$ has a symmetric conjugate $\mM\coloneqq\mD^{-1/2} \mP\mD^{1/2}$ with eigendecomposition $\mM = \mQ \mLambda \mQ^T$ for orthogonal $\mQ$.
Given this decomposition, we can write
\begin{align}
    \mPhi_J^\prime &= \mD^{1/2} \mQ \mLambda^{t_J} \mQ^T \mD^{-1/2}, \nonumber \\
    \mPsi_j^\prime &= \mD^{1/2} \mQ (\mLambda^{t_j} - \mLambda^{t_{j+1}}) \mQ^T \mD^{-1/2}, \quad 0 \leq j \leq J-1, \nonumber
\end{align}
where we set $t_0 = 0$ to simplify notations. Therefore, we have
\begin{equation*}
\|\mPhi_J^\prime \vx \|_{\mD^{-1/2}}^2 = \langle \mPhi_J^\prime \vx, \mPhi_J^\prime \vx \rangle_{\mD^{-1/2}} = \|\mLambda^{t_J} \mQ^T \mD^{-1/2} \vx\|_2^2.
\end{equation*}
If we consider a change of variable to $\vy = \mQ^T \mD^{-1/2} \vx$, we have $\|\vx\|_{\mD^{-1/2}}^2 = \|\mD^{-1/2}\vx\|_2^2 = \|\vy\|_2^2$, while $\|\mPhi_J^\prime \vx \|_{\mD^{-1/2}}^2 = \|\Lambda^{t_J}\vy\|_2^2$. Similarly, we can also reformulate the operations of the other filters in terms of diagonal matrices applied to $\vy$ as $\|\mPsi_j^\prime \vx \|_{\mD^{-1/2}}^2 = \|(\Lambda^{t_j} - \Lambda^{t_{j+1}})\vy\|_2^2$. 

Given these reformulations, we can now write
\begin{multline*}
    \|\Lambda^{t_J}\vy\|_2^2 + \sum_{j=0}^{J-1}\|(\Lambda^{t_j} - \Lambda^{t_{j+1}})\vy\|_2^2 = \\
    \sum_{i=1}^n \vy_i^2 \cdot \left(\lambda^{2 t_J} + \sum\nolimits_{j=0}^{J-1}(\lambda_i^{t_j} - \lambda_i^{t_{j+1}})^2\right).
\end{multline*}
Since $0 \leq \lambda_i \leq 1$ and $0 = t_0 < t_1 < \cdots < t_J$, we have
$$
\lambda_i^{2 t_J} + \sum_{j=0}^{J-1}(\lambda_i^{t_j} - \lambda_i^{t_{j+1}})^2 \leq \left(\lambda_i^{t_J} + \sum_{j=0}^{J-1}\lambda_i^{t_j} - \lambda_i^{t_{j+1}}\right)^2
= 1,
$$
which yields the upper bound $\|\Lambda^{t_J}\vy\|_2^2 + \sum_{j=0}^{J-1}\|(\Lambda^{t_j} - \Lambda^{t_{j+1}})\vy\|_2^2 \leq \|\vy\|_2^2$.
On the other hand, since $t_1 > 0 = t_0$,
$$
\lambda_i^{2 t_J} + \sum_{j=0}^{J-1}(\lambda_i^{t_j} - \lambda_i^{t_{j+1}})^2 \geq \lambda_i^{2 t_J} + (1 - \lambda_i^{t_1})^2,
$$ 
and thus, setting $C \coloneqq \min_{0 \leq \xi \leq 1} (\xi^{2 t_J} + (1 - \xi^{t_1})^2) > 0$,
we get the lower bound $\|\Lambda^{t_J}\vy\|_2^2 + \sum_{j=0}^{J-1}\|(\Lambda^{t_j} - \Lambda^{t_{j+1}})\vy\|_2^2 \geq C \|\vy\|_2^2$. Applying the reverse change of variable to $\vx$ and $L^2(\gG,\mD^{-1/2})$ yields the result of the theorem.
\end{proof}

Intuitively, the upper (i.e., nonexpansive) frame bound implies stability in the sense that small perturbations in the input graph signal will only result in small perturbations in the representation extracted by the constructed filter bank. Further, the lower frame bound ensures certain energy preservation by the constructed filter bank, thus indicating the nonexpansiveness is not implemented in a trivial fashion (e.g., by constant features independent of input signal). 

The following theorem establishes that any such configuration, extracted from $\mathcal{W}^\prime_J$ via Eq.~\ref{eq_scattering (node) features}-\ref{eq_scattering (graph) featrues}, is permutation equivariant at the node-level and permutation invariant at the graph level. This guarantees that the extracted (in this case learned) features indeed encode intrinsic graph geometry rather than a priori indexation.
\begin{thm}\label{thm:permutation}
Let $\mU_p^\prime$ and $\mS_{p,q}^\prime$ be defined as in Eq.~\ref{eq_scattering (node) features} and \ref{eq_scattering (graph) featrues} (correspondingly), with the filters from $\mathcal{W}_J^\prime$ with an arbitrary configuration $0 < t_1 < \cdots < t_J$. Then, for any permutation $\Pi$ over the nodes of $\gG$, and any graph signal $\vx\in L^2(\gG,\mD^{-1/2})$ we have $
\mU_p^\prime \Pi \vx = \Pi \mU_p^\prime \vx$ and $\mS_{p,q}^\prime \Pi \vx = \mS_{p,q}^\prime \vx$, for $p \in \cup_{m \in \N} \N_0^{m}, q \in \N$, where geometric scattering implicitly considers here the node ordering supporting its input signal.
\end{thm}

\begin{proof}
Denote the permutation group on $n$ elements as $S_n$. For a permutation $\Pi \in S_n$ we let $\overline{\gG} = \Pi(\gG)$ be the graph obtained by permuting the vertices of $\gG$ with $\Pi$. The corresponding permutation operation on a graph signal $\vx \in L^2(\gG,\mD^{-1/2})$ gives a signal $\Pi \vx \in L^2(\overline{\gG}, \mD^{-1/2})$, which we implicitly considered in the statement of the theorem, without specifying these notations for simplicity. Rewriting the statement of the theorem more rigorously with the introduced notations, we aim to show that $\overline{\mU}_p^\prime \Pi \vx = \Pi \mU_p^\prime \vx$ and $\overline{\mS}_{p,q}^\prime \Pi \vx = \mS_{p,q}^\prime \vx$ under suitable conditions, where the operation $\mU_p^\prime$ from $\gG$ on the permuted graph $\overline{\gG}$ is denoted here by $\overline{\mU}_p^\prime$ and likewise for $\mS_{p,q}^\prime$ we have $\overline{\mS}_{p,q}^\prime$. 

We start by showing $\mU_p^\prime$ is permutation equivariant. First, we notice that for any $\Psi_j$, $0 < j < J$ we have that $\overline{\Psi}_j \Pi \vx = \Pi \Psi_j \vx$, as for $1 \le j \le J - 1$,
$$
    \overline{\mPsi}_j \Pi \vx 
    = (\Pi \mP^{t_j} \Pi^T - \Pi\mP^{t_{j+1}}\Pi^T) \Pi \vx
    = \Pi \mPsi_j \vx,
$$
with similar reasoning for $j\in \{0, J\}$. Note that the element-wise absolute value yields $\vert \Pi \vx \vert = \Pi \vert \vx \vert$ for any permutation matrix $\Pi$. These two observations inductively yield 
\begin{align*}
    \overline{\mU}_p^\prime \Pi\vx =& \overline{\mPsi}_{j_m}^\prime \vert \overline{\mPsi}_{j_{m-1}}^\prime \dots \vert \overline{\mPsi}_{j_2}^\prime \vert \overline{\mPsi}_{j_1}^\prime \Pi\vx\vert \vert \dots \vert \\
    =&  \overline{\mPsi}_{j_m}^\prime \vert \overline{\mPsi}_{j_{m-1}}^\prime \dots \vert \overline{\mPsi}_{j_2}^\prime \Pi \vert \mPsi_{j_1}^\prime \vx\vert \vert \dots \vert =\dots= \Pi \mU_p^\prime \vx. 
\end{align*}
To show $\mS_{p,q}^\prime$ is permutation invariant, first notice that for any statistical moment $q > 0$, we have  $\vert \Pi \vx \vert^q = \Pi \vert \vx \vert^q$ and further as sums are commutative, $\sum_j (\Pi \vx)_j = \sum_j \vx_j$. We then have
$$
    \overline{\mS}_{p,q}^\prime \Pi \vx = \sum_{i=1}^n \vert \overline{\mU}_p^\prime \Pi \vx [v_i] \vert^q = \sum_{i=1}^n \vert \Pi \mU_p^\prime \vx [v_i] \vert^q
    = \mS_{p,q}^\prime \vx,
$$
which
completes the proof of the theorem.
\end{proof}

We note that the results in Theorems~\ref{thm:frame}-\ref{thm:permutation} and their proofs closely follow the theoretical framework proposed by~\citet{perlmutter2019understanding}. We carefully account here for the relaxed learned configuration, which replaces the original handcrafted one there.

\begin{figure*}[ht]
    \begin{center}
    \includegraphics[width=0.7\linewidth]{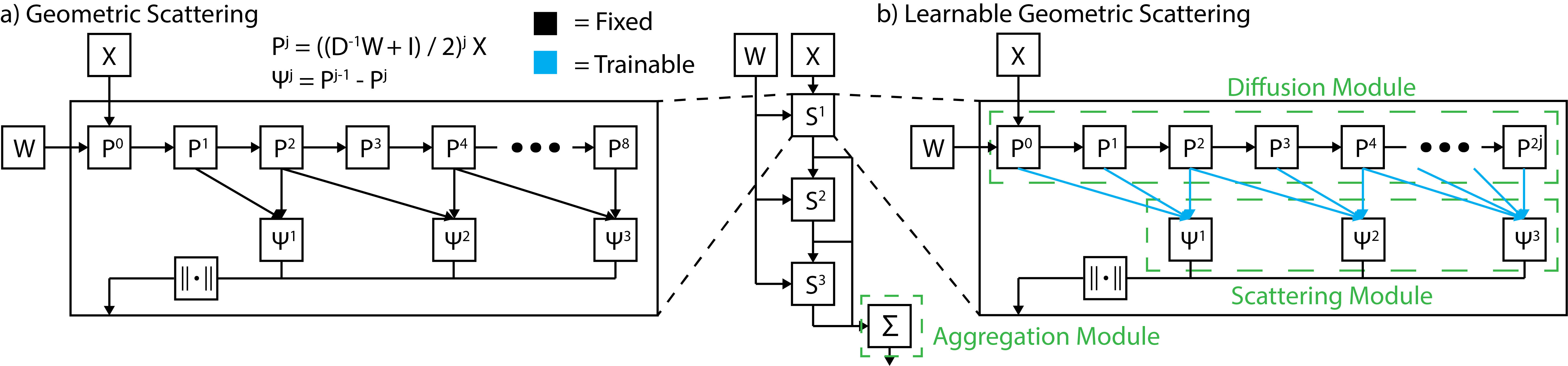}
    \end{center}
    \vspace{-5mm}
    \caption{LEGS module learns to select the appropriate scattering scales from the data.}
    \label{fig:arch}
    \vspace{-5mm}
\end{figure*}
\section{A Learnable Geom. Scattering Module}\label{sect_architecture} 


The adaptive geometric scattering construction presented in Sec.~\ref{sect_theory} is implemented here in a data-driven way via a backpropagation-trainable module. Throughout this section, we consider an input graph signal $\vx\in\R^n$ or, equivalently, a collection of graph signals $\mX\in\R^{n\times N_{\ell-1}}$. The forward propagation of these signals can be divided into three major submodules. First, a {\em diffusion submodule} implements the Markov process that forms the basis of the filter bank and transform. Then, a {\em scattering submodule} implements the filters and the corresponding cascade, while allowing the learning of the scales $t_1,\ldots,t_J$. Finally, the {\em aggregation module} collects the extracted features to provide a graph and produces the task-dependent output.

\vspace{2pt}\noindent\textbf{The diffusion submodule.} We build a set of $m\in\N$ subsequent diffusion steps of the signal $\vx$ by iteratively multiplying the diffusion matrix $\mP$ to the left of the signal, resulting in
$
    \left[ \mP \vx, \mP^2 \vx, \dots, \mP^m \vx \right].
$
Since $\mP$ is often sparse, for efficiency reasons these filter responses are implemented via an RNN structure consisting of $m$ RNN modules. Each module propagates the incoming hidden state $\vh_{t-1}, t = 1, \ldots, m$ with $\mP$ with the readout $\vo_t$ equal to the produced hidden state,
$
    \vh_t \coloneqq \mP \vh_{t-1}, \quad \vo_t \coloneqq \vh_t.
$


\vspace{2pt}\noindent\textbf{The scattering submodule.} Next, we consider the selection of $J\leq m$ diffusion scales for the flexible filter bank construction with wavelets defined according to~Eq.~\ref{eq_wavelet matrix relaxed}. We found this was the most influential part of the architecture. We experimented with methods of increasing flexibility:
\begin{enumerate*}
    \item Selection of $\{t_j\}_{j=1}^{J-1}$ as dyadic scales (as in Sec.~\ref{sect_geometric scattering} and Eq.~\ref{eq_wavelet matrix}), fixed for all datasets (LEGS-FIXED); and
    \item Selection of each $t_j$ using softmax and sorting by $j$, learnable per model (LEGS-FCN and LEGS-RBF, depending on output layer explained below).
\end{enumerate*}

\noindent For the scale selection, we use a selection matrix $\mF\in\R^{J\times m}$, where each row $\mF_{(j,\cdot)}, j = 1,\ldots,J$ is dedicated to identifying the diffusion scale of the wavelet $\mP^{t_j}$ via a one-hot encoding. This is achieved by setting $\mF
    \coloneqq \sigma(\mTheta) = [\sigma(\vtheta_1), \sigma(\vtheta_2), \ldots, \sigma(\vtheta_J)]^T,$
where $\vtheta_j \in\R^m$ constitute the rows of the trainable weight matrix $\mTheta$, and $\sigma$ is the softmax function. While this construction may not strictly guarantee an exact one-hot encoding, we assume that the softmax activations yield a sufficient approximation. Further, without loss of generality, we assume that the rows of $\mF$ are ordered according to the position of the leading ``one'' activated in every row. In practice, this can be easily enforced by reordering the rows. We now construct the filter bank $\widetilde{\mathcal{W}}_{\mF} \coloneqq\{\widetilde\mPsi_{j}, \widetilde\mPhi_{J}\}_{j=0}^{J-1}$ with the filters
\begin{align}\label{eq_wavelet matrix relaxed}
    \widetilde{\mPsi}_{0} \vx &= \vx - \sum_{t=1}^m \mF_{(1,t)} \mP^t \vx, \quad \widetilde{\mPhi}_{J} \vx = \sum_{t=1}^m \mF_{(J,t)} \mP^t \vx, \\
    \widetilde{\mPsi}_{j} \vx &= \sum_{t=1}^m \left[\mF_{(j,t)} \mP^t \vx - \mF_{(j+1,t)} \mP^t \vx \right], \quad 1 \leq j \leq J-1, \nonumber
\end{align}
matching and implementing the construction of $\mathcal{W}_J^\prime$ (Eq.~\ref{eq_adaptive wavelet matrices}).
While many approaches may be applied to aggregate node-level features into graph-level features such as max, mean, sum pooling, and the more powerful TopK~\citep{gao_graph_2019-1} or attention pooling~\citep{velickovic_graph_2018}, we follow the statistical-moment aggregation explained in Secs.~\ref{sect_geometric scattering}-\ref{sect_theory} motivated by \citet{gao2019geometric,perlmutter2019understanding} and leave exploration of other pooling methods to future work.



\subsection{Incorporating LEGS into a larger neural network} 

As shown in~\citet{gao2019geometric} on graph classification, this aggregation works particularly well in conjunction with support vector machines (SVMs) based on the radial basis function (RBF) kernel. Here, we consider two configurations for the task-dependent output layer of the network, either using two fully connected layers after the learnable scattering layers, which we denote LEGS-FCN, or using a modified RBF network~\citep{broomhead_radial_1988}, which we denote LEGS-RBF, to produce the final classification. 

The latter configuration more accurately processes scattering features as shown in Table~\ref{tab:full_table}. Our RBF network works by first initializing a fixed number of movable anchor points. Then, for every point, new features are calculated based on the radial distances to these anchor points. In previous work on radial basis networks these anchor points were initialized independent of the data. We found that this led to training issues if the range of the data was not similar to the initialization of the centers. Instead, we first use a batch normalization layer to constrain the scale of the features and then pick anchors randomly from the initial features of the first pass through our data. This gives an RBF-kernel network with anchors that are always in the range of the data. Our RBF layer is then $\text{RBF}(\vx) = \phi(\| \text{BatchNorm}(\vx) - \vc\|)$ with $\phi(\vx) = e^{-\|\vx\|^2}$.

\section{Empirical Results}\label{sect_results}
We investigate the LEGS module on whole graph classification and graph regression tasks that arise in a variety of contexts, with emphasis on the more complex biochemical datasets. Unlike other types of data, biochemical graphs do not exhibit the small-world structure of social graphs and may have large graph diameters for their size. Further, the connectivity patterns of biomolecules are very irregular due to 3D folding and long-range connections, and thus ordinary local node aggregation methods may miss such connectivity differences.
\subsection{Whole Graph Classification}
We perform whole graph classification by using eccentricity (max distance of a node to other nodes) and clustering coefficient (percentage of links between the neighbors of the node compared to a clique) as node features as are used in \citet{gao2019geometric}. We compare against graph convolutional networks (GCN)~\citep{kipf_semi-supervised_2016}, GraphSAGE~\citep{hamilton_inductive_2017}, graph attention network (GAT)~\citep{velickovic_graph_2018}, graph isomorphism network (GIN)~\citep{xu_how_2019}, Snowball network~\citep{luan_break_2019}, and fixed geometric scattering with a support vector machine classifier (GS-SVM) as in \citet{gao2019geometric}, and a baseline which is a 2-layer neural network on the features averaged across nodes (disregarding graph structure). These comparisons are meant to inform when including learnable graph scattering features are helpful in extracting whole graph features. Specifically, we are interested in the types of graph datasets where existing graph neural network performance can be improved upon with scattering features. We evaluate these methods across seven biochemical and six social network datasets for graph classification with hundreds to thousands of graphs and tens to hundreds of nodes.

\begin{table*}[t]
\caption{Mean $\pm$ std. over 10 test sets on biochemical (top) and social network (bottom) datasets.}
 \label{tab:full_table}
\centering
\scalebox{0.7}{
        \begin{tabular}{llllllllll}
\toprule
{} &                   LEGS-RBF &                   LEGS-FCN &         LEGS-FIXED &                        GCN &                  GraphSAGE &               GAT &                        GIN &                     GS-SVM &                   Baseline \\
\midrule
DD               &           72.58 $\pm$ 3.35 &           72.07 $\pm$ 2.37 &   69.09 $\pm$ 4.82 &           67.82 $\pm$ 3.81 &           66.37 $\pm$ 4.45 &  68.50 $\pm$ 3.62 &           42.37 $\pm$ 4.32 &           72.66 $\pm$ 4.94 &  \textbf{75.98 $\pm$ 2.81} \\
ENZYMES          &           36.33 $\pm$ 4.50 &  \textbf{38.50 $\pm$ 8.18} &   32.33 $\pm$ 5.04 &           31.33 $\pm$ 6.89 &           15.83 $\pm$ 9.10 &  25.83 $\pm$ 4.73 &           36.83 $\pm$ 4.81 &           27.33 $\pm$ 5.10 &           20.50 $\pm$ 5.99 \\
MUTAG            &           33.51 $\pm$ 4.34 &           82.98 $\pm$ 9.85 &  81.84 $\pm$ 11.24 &           79.30 $\pm$ 9.66 &          81.43 $\pm$ 11.64 &  79.85 $\pm$ 9.44 &           83.57 $\pm$ 9.68 &  \textbf{85.09 $\pm$ 7.44} &           79.80 $\pm$ 9.92 \\
NCI1             &  \textbf{74.26 $\pm$ 1.53} &           70.83 $\pm$ 2.65 &   71.24 $\pm$ 1.63 &           60.80 $\pm$ 4.26 &           57.54 $\pm$ 3.33 &  62.19 $\pm$ 2.18 &           66.67 $\pm$ 2.90 &           69.68 $\pm$ 2.38 &           56.69 $\pm$ 3.07 \\
NCI109           &  \textbf{72.47 $\pm$ 2.11} &           70.17 $\pm$ 1.46 &   69.25 $\pm$ 1.75 &           61.30 $\pm$ 2.99 &           55.15 $\pm$ 2.58 &  61.28 $\pm$ 2.24 &           65.23 $\pm$ 1.82 &           68.55 $\pm$ 2.06 &           57.38 $\pm$ 2.20 \\
PROTEINS         &           70.89 $\pm$ 3.91 &           71.06 $\pm$ 3.17 &   67.30 $\pm$ 2.94 &           74.03 $\pm$ 3.20 &           71.87 $\pm$ 3.50 &  73.22 $\pm$ 3.55 &  \textbf{75.02 $\pm$ 4.55} &           70.98 $\pm$ 2.67 &           73.22 $\pm$ 3.76 \\
PTC              &  \textbf{57.26 $\pm$ 5.54} &           56.92 $\pm$ 9.36 &   54.31 $\pm$ 6.92 &          56.34 $\pm$ 10.29 &           55.22 $\pm$ 9.13 &  55.50 $\pm$ 6.90 &           55.82 $\pm$ 8.07 &           56.96 $\pm$ 7.09 &           56.71 $\pm$ 5.54 \\
\midrule
COLLAB           &           75.78 $\pm$ 1.95 &           75.40 $\pm$ 1.80 &   72.94 $\pm$ 1.70 &           73.80 $\pm$ 1.73 &  \textbf{76.12 $\pm$ 1.58} &  72.88 $\pm$ 2.06 &           62.98 $\pm$ 3.92 &           74.54 $\pm$ 2.32 &           64.76 $\pm$ 2.63 \\
IMDB-BINARY      &           64.90 $\pm$ 3.48 &           64.50 $\pm$ 3.50 &   64.30 $\pm$ 3.68 &           47.40 $\pm$ 6.24 &           46.40 $\pm$ 4.03 &  45.50 $\pm$ 3.14 &           64.20 $\pm$ 5.77 &  \textbf{66.70 $\pm$ 3.53} &           47.20 $\pm$ 5.67 \\
IMDB-MULTI       &           41.93 $\pm$ 3.01 &           40.13 $\pm$ 2.77 &   41.67 $\pm$ 3.19 &           39.33 $\pm$ 3.13 &           39.73 $\pm$ 3.45 &  39.73 $\pm$ 3.61 &           38.67 $\pm$ 3.93 &  \textbf{42.13 $\pm$ 2.53} &           39.53 $\pm$ 3.63 \\
REDDIT-BINARY    &  \textbf{86.10 $\pm$ 2.92} &           78.15 $\pm$ 5.42 &   85.00 $\pm$ 1.93 &           81.60 $\pm$ 2.32 &           73.40 $\pm$ 4.38 &  73.35 $\pm$ 2.27 &           71.40 $\pm$ 6.98 &           85.15 $\pm$ 2.78 &           69.30 $\pm$ 5.08 \\
REDDIT-MULTI-12K &           38.47 $\pm$ 1.07 &           38.46 $\pm$ 1.31 &   39.74 $\pm$ 1.31 &  \textbf{42.57 $\pm$ 0.90} &           32.17 $\pm$ 2.04 &  32.74 $\pm$ 0.75 &           24.45 $\pm$ 5.52 &           39.79 $\pm$ 1.11 &           22.07 $\pm$ 0.98 \\
REDDIT-MULTI-5K  &           47.83 $\pm$ 2.61 &           46.97 $\pm$ 3.06 &   47.17 $\pm$ 2.93 &  \textbf{52.79 $\pm$ 2.11} &           45.71 $\pm$ 2.88 &  44.03 $\pm$ 2.57 &           35.73 $\pm$ 8.35 &           48.79 $\pm$ 2.95 &           36.41 $\pm$ 1.80 \\
\bottomrule
\end{tabular}
    }
\vspace{-3mm}
\end{table*}

\noindent\textbf{LEGS outperforms on biochemical datasets.} Most work on graph neural networks has focused on social networks which have a well-studied structure. However, biochemical graphs that represent molecules and tend to be overall smaller and less connected than social networks~\cite{gao2019geometric}. In particular, we find that LEGS outperforms other methods by a significant margin on biochemical datasets with relatively small but high diameter graphs (NCI1, NCI109, ENZYMES, PTC), as shown in Tab.~\ref{tab:full_table}. On extremely small graphs we find that GS-SVM performs best, which is expected as other methods with more parameters can easily overfit the data. We reason that the performance increase exhibited by LEGS module networks, and to a lesser extent GS-SVM, on these biochemical graphs is due the ability of geometric scattering to compute complex connectivity features via its multiscale diffusion wavelets. Thus, methods that rely on a scattering construction would in general perform better, with the flexibility and trainability of the LEGS module giving it an edge on most tasks. Additionally, LEGS performs well on social network datasets
In Tab.~\ref{tab:full_table}, we see that on the social network datasets LEGS performs well. Overlooking the fixed scattering transform GS-SVM, which was tuned in \citet{gao2019geometric} with a focus on these particular social network datasets, a LEGS module architecture is best on three out of the six social datasets and second best on the other three. 

\begin{table}[ht]
\caption{CASP GDT regression error over three seeds}
\centering
\begin{small}
\begin{tabular}{lll}
\toprule
($\mu \pm \sigma$) &                   Train MSE &                     Test MSE \\
\midrule
LEGS-FCN &  \textbf{134.34 $\pm$ 8.62} &  \textbf{144.14 $\pm$ 15.48} \\
LEGS-RBF  &           140.46 $\pm$ 9.76 &           152.59 $\pm$ 14.56 \\
LEGS-FIXED  &          136.84 $\pm$ 15.57 &            160.03 $\pm$ 1.81 \\
GCN       &          289.33 $\pm$ 15.75 &           303.52 $\pm$ 18.90 \\
GraphSAGE &          221.14 $\pm$ 42.56 &           219.44 $\pm$ 34.84 \\
GIN       &          221.14 $\pm$ 42.56 &           219.44 $\pm$ 34.84 \\
Baseline  &           393.78 $\pm$ 4.02 &           402.21 $\pm$ 21.45 \\
\bottomrule
\end{tabular}
\end{small}
\label{tab:casp}
\vspace{-5mm}
\end{table}

\begin{figure}[ht]
    \centering
    \includegraphics[width=0.8\linewidth]{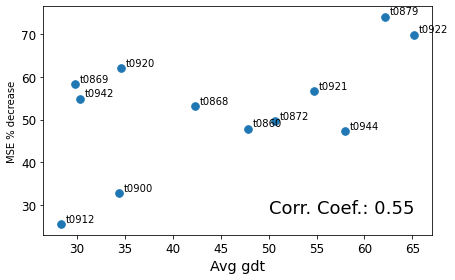}
    \vspace{-3mm}
    \caption{CASP dataset LEGS-FCN \% improvement over GCN in MSE of GDT prediction vs. Average GDT score.}
    \label{fig:casp}
\end{figure}

\begin{table}
    \caption{Mean $\pm$ std.\ over four runs of mean squared error over 19 targets for the QM9 dataset, lower is better.}
    \centering

\adjustbox{height=0.119\linewidth}{
\begin{tabular}{ll}
\toprule
$(\mu \pm \sigma)$ & Test MSE                    \\
\midrule
LEGS-FCN    &  \textbf{0.216 $\pm$ 0.009} \\
LEGS-FIXED  &  0.228 $\pm$ 0.019 \\
GraphSAGE   &  0.524 $\pm$ 0.224 \\
\bottomrule
\end{tabular}
}~\adjustbox{height=0.119\linewidth}{
\begin{tabular}{ll}
\toprule
$(\mu \pm \sigma)$ & Test MSE                    \\
\midrule
GCN         &  0.417 $\pm$ 0.061 \\
GIN         &  0.247 $\pm$ 0.037 \\
Baseline    &  0.533 $\pm$ 0.041 \\
\bottomrule
\end{tabular}
}
    \label{tab:qm9}
    \vspace{-5mm}
\end{table}

\subsection{Graph Regression}
We next evaluate learnable scattering on two graph regression tasks, the QM9~\citep{gilmer_neural_2017} graph regression dataset, and a new task from the critical assessment of structure prediction (CASP) challenge~\citep{moult_critical_2018}. In the CASP task, the main objective is to score protein structure prediction/simulation models in terms of the discrepancy between their predicted structure and the actual structure of the protein (which is known a priori). The accuracy of such 3D structure predictions are evaluated using a variety of metrics, but we focus on the global distance test (GDT) score~\citep{modi_assessment_2016}. The GDT score measures the similarity between tertiary structures of two proteins with amino-acid correspondence. A higher score means two structures are more similar. For a set of predicted 3D structures for a protein, we would like to quantify their quality by the GDT score.

For this task we use the CASP12 dataset~\citep{moult_critical_2018} and preprocess it similarly to \citet{ingraham_generative_2019}, creating a KNN graph between proteins based on 3D coordinates of each amino acid. From this KNN graph we regress against the GDT score. We evaluate on 12 proteins from the CASP12 dataset and choose random (but consistent) splits with 80\% train, 10\% validation, and 10\% test data out of 4000 total structures. We are interested in structure similarity and use no nonstructural node features.

\noindent\textbf{LEGS outperforms on all CASP targets.} Across all CASP targets we find that LEGS-based architectures significantly outperforms GNN and baseline models. This performance improvement is particularly stark on the easiest structures (measured by average GDT) but is consistent across all structures. In Fig.~\ref{fig:casp} we show the relationship between percent improvement of LEGS over the GCN model and the average GDT score across the target structures. We draw attention to target t0879, where LEGS shows the greatest improvement over other methods. Interestingly this target has particularly long-range dependencies~\citep{ovchinnikov_protein_2018}. Since other methods are unable to model these long-range connections, this suggests LEGS is particularly important on these more difficult to model targets.

\vspace{2pt}\noindent\textbf{LEGS outperforms on the QM9 dataset.} We evaluate the performance of LEGS-based networks on the quantum chemistry dataset QM9~\citep{gilmer_neural_2017}, which consists of 130,000 molecules with $\sim$18 nodes per molecule. We use the node features from \citet{gilmer_neural_2017}, with the addition of eccentricity and clustering coefficient features, and ignore the edge features. We whiten all targets to have zero mean and unit standard deviation. We train each network against all 19 targets and evaluate the mean squared error on the test set with mean and std.\ over four runs finding that LEGS improves over existing models (Tab.~\ref{tab:qm9}). 

\section{Conclusion}

Here we introduced a flexible geometric scattering module, that serves as an alternative to standard graph neural network architectures and is capable of learning rich multi-scale features. Our learnable geometric scattering module allows a task-dependent network to choose the appropriate scales of the multiscale graph diffusion wavelets that are part of the geometric scattering transform. We show that incorporation of this module yields improved performance on graph classification and regression tasks, particularly on biochemical datasets, while keeping strong guarantees on extracted features. This also opens the possibility to provide additional flexibility to the module to enable node-specific or graph-specific tuning via attention mechanisms, which are an exciting future direction, but out of scope for the current work.

\bibliographystyle{IEEEbib}
\bibliography{mlsp}

\begin{thebibliography}{10}

\bibitem{bronstein_geometric_2017}
M.~Bronstein, J.~Bruna, Y.~LeCun, A.~Szlam, and P.~Vandergheynst,
\newblock ``Geometric deep learning: Going beyond {{Euclidean}} data,''
\newblock {\em IEEE Signal Process. Mag.}, 2017.

\bibitem{kipf_semi-supervised_2016}
T.~Kipf and M.~Welling,
\newblock ``Semi-supervised classification with graph convolutional networks,''
\newblock {\em Proc. of ICLR}, 2016.

\bibitem{hamilton_inductive_2017}
W.~Hamilton, R.~Ying, and J.~Leskovec,
\newblock ``Inductive representation learning on large graphs,''
\newblock {\em Adv. in NeurIPS 30}, 2017.

\bibitem{xu_how_2019}
K.~Xu, W.~Hu, J.~Leskovec, and S.~Jegelka,
\newblock ``How powerful are graph neural networks?,''
\newblock {\em Proc. of ICLR}, 2019.

\bibitem{li2018deeper}
Q.~Li, Z.~Han, and X.~Wu,
\newblock ``Deeper insights into graph convolutional networks for
  semi-supervised learning,''
\newblock in {\em Proc. of AAAI}, 2018.

\bibitem{alon2020bottleneck}
U.~Alon and E.~Yahav,
\newblock ``On the bottleneck of graph neural networks and its practical
  implications,''
\newblock in {\em Proc. of ICLR}, 2021.

\bibitem{barcelo2020logical}
P.~Barcel{\'o}, E.~Kostylev, M.~Monet, J.~P{\'e}rez, J.~Reutter, and J.~Silva,
\newblock ``The logical expressiveness of graph neural networks,''
\newblock in {\em Proc. of ICLR}, 2020.

\bibitem{mallat_group_2012}
S.~Mallat,
\newblock ``Group invariant scattering,''
\newblock {\em Commun. Pure Appl. Math.}, 2012.

\bibitem{gao2019geometric}
F.~Gao, G.~Wolf, and M.~Hirn,
\newblock ``Geometric scattering for graph data analysis,''
\newblock in {\em Proc. of ICML}, 2019.

\bibitem{gama2019diffusion}
F.~Gama, A.~Ribeiro, and J.~Bruna,
\newblock ``Diffusion scattering transforms on graphs,''
\newblock in {\em Proc. of ICLR}, 2019.

\bibitem{zou_graph_2019}
D.~Zou and G.~Lerman,
\newblock ``Graph convolutional neural networks via scattering,''
\newblock {\em Appl. Comp. Harm. Anal.}, 2019.

\bibitem{perlmutter_geometric_2018}
M.~Perlmutter, G.~Wolf, and M.~Hirn,
\newblock ``Geometric scattering on manifolds,''
\newblock in {\em NeurIPS 2018 Workshop on Integration of Deep Learning
  Theories}, 2018.

\bibitem{coifman_diffusion_2006-1}
R.~Coifman and M.~Maggioni,
\newblock ``Diffusion wavelets,''
\newblock {\em Appl. Comp. Harm. Anal.}, vol. 21(1), pp. 53--94, 2006.

\bibitem{gama_stability_2019}
F.~Gama, J.~Bruna, and A.~Ribeiro,
\newblock ``Stability of graph scattering transforms,''
\newblock {\em Adv. in NeurIPS 32}, 2019.

\bibitem{perlmutter2019understanding}
M.~Perlmutter, F.~Gao, G.~Wolf, and M.~Hirn,
\newblock ``Understanding graph neural networks with asymmetric geometric
  scattering transforms,''
\newblock arXiv:1911.06253, 2019.

\bibitem{gao_graph_2019-1}
H.~Gao and S.~Ji,
\newblock ``Graph {U}-{Nets},''
\newblock in {\em Proc. of ICML}, 2019, vol.~97 of {\em PMLR}, pp. 2083--2092.

\bibitem{velickovic_graph_2018}
P.~Veli{\v c}kovi{\'c}, G.~Cucurull, A.~Casanova, A.~Romero, P.~Li{\`o}, and
  Y.~Bengio,
\newblock ``Graph attention networks,''
\newblock {\em Proc. of ICLR}, 2018.

\bibitem{broomhead_radial_1988}
D.~Broomhead and D.~Lowe,
\newblock ``Multivariable functional interpolation and adaptive networks,''
\newblock {\em Complex Systems}, vol. 2, pp. 321--355, 1988.

\bibitem{luan_break_2019}
S.~Luan, M.~Zhao, X.~Chang, and D.~Precup,
\newblock ``Break the ceiling: Stronger multi-scale deep graph convolutional
  networks,''
\newblock in {\em Adv. in NeurIPS 32}, 2019.

\bibitem{gilmer_neural_2017}
J.~Gilmer, S.~Schoenholz, P.~Riley, O.~Vinyals, and G.~Dahl,
\newblock ``Neural message passing for quantum chemistry,''
\newblock in {\em Proc. of ICML}, 2017, vol.~70 of {\em PMLR}, pp. 1263--1272.

\bibitem{moult_critical_2018}
J.~Moult, K.~Fidelis, A.~Kryshtafovych, T.~Schwede, and A.~Tramontano,
\newblock ``Critical assessment of methods of protein structure prediction
  ({{CASP}})\textemdash{{Round XII}},''
\newblock {\em Proteins Struct. Funct. Bioinforma.}, 2018.

\bibitem{modi_assessment_2016}
V.~Modi, Q.~Xu, S.~Adhikari, and R.~Dunbrack,
\newblock ``Assessment of {{Template}}-{{Based Modeling}} of {{Protein
  Structure}} in {{CASP11}},''
\newblock {\em Proteins}, 2016.

\bibitem{ingraham_generative_2019}
J.~Ingraham, V.~Garg, R.~Barzilay, and T.~Jaakkola,
\newblock ``Generative models for {{Graph}}-{{Based Protein Design}},''
\newblock in {\em Adv. in NeurIPS 32}, 2019.

\bibitem{ovchinnikov_protein_2018}
S.~Ovchinnikov, H.~Park, D.~Kim, F.~DiMaio, and D.~Baker,
\newblock ``Protein structure prediction using {{Rosetta}} in {{CASP12}},''
\newblock {\em Proteins Struct. Funct. Bioinforma.}, 2018.

\end{thebibliography}

\end{document}